%% file: main.tex
\documentclass[sigconf]{acmart}

\usepackage{booktabs} 

\usepackage{algorithm}
\usepackage{algorithmicx}
\usepackage[noend]{algpseudocode}
\usepackage{amsmath}\allowdisplaybreaks
\usepackage{amssymb}
\usepackage{amsthm}
\usepackage{bbm}
\usepackage{bm}
\usepackage[skip=0.03in]{caption}
\usepackage{color}
\usepackage{dsfont}
\usepackage{enumerate}
\usepackage{graphicx}
\usepackage{subfigure}
\usepackage[normalem]{ulem}
\newtheorem{theorem}{Theorem}[section]
\usepackage[backgroundcolor=White]{todonotes}
\usepackage[capitalize,noabbrev]{cleveref}
\newcommand{\commentout}[1]{}
\newcommand{\junk}[1]{}
\newcommand{\etal}{\emph{et al.}}

\newcommand{\ba}{\boldsymbol{a}}

\newcommand{\bw}{\boldsymbol{w}}
\newcommand{\bx}{\boldsymbol{x}}
\newcommand{\bA}{\boldsymbol{A}}

\newcommand{\bS}{\boldsymbol{S}}
\newcommand{\bone}{\boldsymbol{1}}

\newcommand{\cH}{\mathcal{H}}

\newcommand{\EE}{\mathbb{E}}

\newcommand{\RR}{\mathbb{R}}
\newcommand{\argmax}{\mathrm{argmax}}

\newcommand{\bOne}{\mathds{1}}
\newcommand{\abs}[1]{\left| #1 \right|}
\newcommand{\pihat}{\hat{\pi}}
\newcommand{\vhat}{\hat{V}}
\newcommand{\htilde}{\tilde{h}}
\newcommand{\si}{\text{I}}
\newcommand{\sip}{\text{IP}}
\newcommand{\spbm}{\text{PBM}}
\newcommand{\sr}{\text{R}}
\newcommand{\subl}{\text{L}}

\renewcommand\footnotetextcopyrightpermission[1]{} 

\setcopyright{none}

\begin{document}

\title{Offline Evaluation of Ranking Policies with Click Models}

\author{Shuai Li}
\affiliation{\institution{The Chinese University of Hong Kong}}
\email{shuaili@cse.cuhk.edu.hk}

\author{Yasin Abbasi-Yadkori}
\affiliation{\institution{Adobe Research}}
\email{abbasiya@adobe.com}

\author{Branislav Kveton}
\affiliation{\institution{Adobe Research}}
\email{kveton@adobe.com}

\author{S. Muthukrishnan}
\affiliation{\institution{Rutgers University}}
\email{muthu@cs.rutgers.edu}

\author{Vishwa Vinay}
\affiliation{\institution{Adobe Research}}
\email{vinay@adobe.com}

\author{Zheng Wen}
\affiliation{\institution{Adobe Research}}
\email{zwen@adobe.com}

\renewcommand{\shortauthors}{Shuai Li et al.}

\begin{abstract}
Many web systems rank and present a list of items to users, from recommender systems to search and advertising. An important problem in practice is to evaluate new ranking policies offline and optimize them before they are deployed. We address this problem by proposing evaluation algorithms for estimating the expected number of clicks on ranked lists from historical logged data. The existing algorithms are not guaranteed to be statistically efficient in our problem because the number of recommended lists can grow exponentially with their length. To overcome this challenge, we use models of user interaction with the list of items, the so-called click models, to construct estimators that learn statistically efficiently. We analyze our estimators and prove that they are more efficient than the estimators that do not use the structure of the click model, under the assumption that the click model holds. We evaluate our estimators in a series of experiments on a real-world dataset and show that they consistently outperform prior estimators.
\end{abstract}

%
%


\maketitle

\input{intro}
\input{setting}
\input{estimators}
\input{analysis}
\input{generalization}
\input{experiments}
\input{related_work}

\input{conclusions}

\bibliographystyle{ACM-Reference-Format}
\bibliography{ref}

\end{document}

%% file: intro.tex

\section{Introduction}
\label{sec:introduction}

Many web applications, including search, advertising, and recommender systems, generate ranked lists of items and present them to users. Items can be web pages, movies, or products. The industry standard for evaluating the quality of recommended lists is online \emph{A/B testing} \cite{siroker2013b}. Because A/B testing may impact the experience of users, it is typically used only as the final validation step, while \emph{offline evaluation} is employed in earlier stages \cite{chuklin2013click}. The benefit of offline evaluation is that poor new policies can be identified before they are deployed, and in turn A/B testing can be done more safely and intelligently.

We study the problem of offline evaluation for estimating the expected number of clicks on lists generated by some policy $h$. This evaluation is done with respect to a previously \emph{logged dataset $S$}, which records user interactions with lists generated by a \emph{logging production policy} $\pi$ \cite{langford2008exploration,strehl2010learning,li2011unbiased}. Existing algorithms for evaluating $h$ are not guaranteed to be statistically efficient in our setting because they look for exact matches of lists in $S$. Roughly speaking, since they rely on variants of importance sampling on lists and the number of lists can be exponential in their length, these algorithms may need exponentially many samples from $\pi$ to perform well.

To overcome this problem of statistical inefficiency, we make structural assumptions on how clicks are generated. In particular, we propose novel unbiased estimators for the expected number of clicks on a list based on well-known click models \cite{chuklin2015click}. The \emph{click model} is a model of user interaction with a ranked list of items, and how the user clicks on these items. Many click models have been proposed, and they represent a spectrum of complexity and accuracy. To illustrate the gain in statistical efficiency from using click models, suppose that all items attract the user independently and that the probability of clicking on an item depends only on its identity. Then it is more efficient to construct a separate estimator for the probability of clicking on each item and then combine these estimators to estimate the expected number of clicks on any given list. Such an estimator would be valid as long as our assumed click model holds, and only requires historical click data.

This paper makes four major contributions:
\begin{enumerate}
\item We formulate the problem of offline evaluation of ranked lists under different click models.

\item We propose clipped importance sampling estimators for a range of click models, from simple where the clicks are independent of both the item and position, to more realistic where the clicks depend on both the item and its position. All estimators are simple and can be computed with a single pass over historical logged data.

\item We analyze the properties of our estimators. We show that they have a lower bias than those that ignore the structure of the list, and that the best policy under our estimators has a higher lower bound on its value. The analysis is under the assumption that our modeling assumptions hold.

\item We evaluate our estimators on a large-scale real-world click dataset. A series of experiments shows that our estimators are consistently better than clipped importance sampling at the list level.
\end{enumerate}

We adopt the following notation. Random variables are denoted by boldface letters, lists by uppercase $A$, and items in the list by lowercase $a$.

%% file: setting.tex

\section{Setting}
\label{sec:setting}

Let $E = \{1, \dots, L\}$ be a \emph{ground set} of $L$ \emph{items}, such as web pages. The items are presented to the user in a \emph{list} of length $K$. Formally, the list is a \emph{$K$-permutation} of $E$, which is chosen from set
\begin{align*}
  \Pi_K(E) = \{(a_1, \dots, a_K): a_1, \dots, a_K \in E; \, a_i \neq a_j \text{ for any } i \neq j\}.
\end{align*}
We assume the following general model of user interaction with a list of items. The responses of the user depend on \emph{context} $x \in X$, which is drawn from some distribution over all contexts $X$. More specifically, let $D(\cdot \mid x)$ be the conditional probability distribution over $\{0, 1\}^{|E| \times K}$ given context $x$. Then a sample from this distribution, $\bw \sim D(\cdot \mid x)$, is a matrix where $\bw(a, k)$ indicates that the user would have \emph{clicked} on item $a$ at position $k$. The expected value of $\bw$ given $x$, $\bar{w}(\cdot \mid x) = \EE_{\bw \sim D(\cdot \mid x)}[\bw]$, is a matrix of conditional click probabilities given $x$. We refer to $\bar{w}(a, k \mid x)$ as the \emph{probability of clicking} on item $a$ at position $k$ given $x$.

A \emph{policy} $\pi$ is a conditional probability distribution of a list given context $x$. We denote this distribution by $\pi(\cdot \mid x)$. The policy interacts with the environment as follows. At time $t$, the environment draws \emph{context} $\bx_t$ and \emph{click realizations} $\bw_t \sim D(\cdot \mid \bx_t)$. The policy observes $\bx_t$ and selects list $\bA_t = (\ba^t_1, \dots, \ba^t_K) \in \prod_K(E)$ according to $\pi(\cdot \mid \bx_t)$, where $\ba^t_k$ is the item at position $k$ at time $t$. Finally, the environment reveals the vector of \emph{item rewards} $(\bw_t(\ba^t_k, k))_{k = 1}^K$, one entry for each displayed item and its position. The other entries of $\bw_t$ are unobserved. The \emph{reward} of list $\bA_t$ is the sum of the observed entries of $\bw_t$. We define it as $f(\bA_t, \bw_t)$, where
\begin{align*}
  f(A, w) = \sum_{k = 1}^K w(a_k, k)
\end{align*}
for any list $A = (a_1, \dots, a_K)$ and $w \in [0, 1]^{|E| \times K}$. It follows that the \emph{expected reward} of list $A$ in context $x$ is
\begin{align*}
  \EE_{\bw \sim D(\cdot \mid x)} [f(A, \bw)] = f(A, \bar{w}(\cdot \mid x))\,.
\end{align*}
Let $\pi$ be a \emph{logging policy}, which interacts with the environment in $n$ steps and generates a \emph{logged dataset}
\begin{align}
  S = \{(x_t, A_t, w_t)\}_{t = 1}^n\,,
  \label{eq:logged dataset}
\end{align}
where $A_t = (a^t_1, \dots, a^t_K)$ and $w_t(a, k)$ is observed only if $a = a^t_k$. We define the \emph{value of policy $h$} as
\begin{align*}
  V(h)
  & = \EE_{\bx} \left[ \EE_{\bA \sim h(\cdot \mid \bx), \bw \sim D(\cdot \mid \bx)} [f(\bA,\bw)] \right] \\
  & = \EE_{\bx} \left[ \EE_{\bA \sim h(\cdot \mid \bx)} f(\bA, \bar{w}(\cdot \mid \bx)) \right]\,.
\end{align*}
Our objective is to estimate $V(h)$ from the logged dataset $S$ in \eqref{eq:logged dataset}. It is common in practice that the logging policy $\pi$ is unknown and has to be estimated \cite{strehl2010learning}. We denote the \emph{estimated logging policy} by $\pihat$, and the corresponding conditional probability distribution of a list given context $x$ by $\pihat(\cdot \mid x)$.

%% file: estimators.tex

\section{Estimators}
\label{sec:estimators}

In this section, we introduce estimators of $V(h)$ that are motivated by the models of user behavior with a displayed list of items, the so-called click models \cite{chuklin2015click}. We propose multiple estimators, from simple to relatively sophisticated, each mirroring a commonly used click model. Simpler estimators are expected to generalize better, under the assumption that the corresponding click model holds.

\subsection{List Estimator}
\label{sec:list estimator}

We start with a list-level estimator, which does not leverage the structure of lists and serves as a baseline for our proposed estimators. Recall that $\pi(A \mid x)$ is the probability that policy $\pi$ chooses list $A$ in response to context $x$.

Let $h(\cdot \mid x)$ be absolutely continuous with respect to $\pi(\cdot \mid x)$, that is $h(A \mid x) = 0$ when $\pi(A \mid x) = 0$. Then
\begin{align}
  V(h)
  & = \EE_{\bx} \left[\EE_{\bA \sim h(\cdot \mid \bx), \bw \sim D(\cdot \mid \bx)} [f(\bA, \bw)]\right]
  \notag \\
  & = \EE_{\bx} \left[\EE_{\bA \sim \pi(\cdot \mid \bx), \bw \sim D(\cdot \mid \bx)}
  \left[f(\bA, \bw) \frac{h(\bA \mid \bx)}{\pi(\bA \mid \bx)}\right]\right]\,,
  \label{eq:absCon}
\end{align}
where the second equality is from the absolute continuity of $h(\cdot \mid x)$ and $h(A \mid x) / \pi(A \mid x)$ is the \emph{importance weight}. The above change-of-measure trick is known as \emph{importance sampling} \cite{andrieu03introduction} and we use it in many forms throughout this paper.

The issue with importance sampling is that $h(A \mid x) / \pi(A \mid x)$ can take large values, which affects the variance of $V(h)$. Therefore, importance sampling estimators are clipped in practice \cite{strehl2010learning,bottou2013counterfactual}. In this work, we define the \emph{list estimator} as
\begin{align*}
  \vhat_\subl(h) =
  \frac{1}{|S|} \sum_{(x, A, w) \in S} f(A, w) \min \left\{\frac{h(A \mid x)}{\pihat(A \mid x)}, M\right\}
\end{align*}
for any estimated logging policy $\pihat$, logged dataset $S$, and \emph{clipping constant} $M > 0$. Roughly speaking, the value of policy $h$ on logged dataset $S$ is the sum of clicks on logged lists scaled by their importance weights.

The value of $M$ trades off the bias and variance of the estimator. As $M \to \infty$, the estimator becomes unbiased but its variance may be huge. As $M \to 0$, the variance of the estimator approaches zero because $\vhat_\subl(h) \to 0$ for any logged dataset $S$.

\subsection{Item-Position (IP) Click Model}
\label{sec:item-position estimator}

A popular assumption in click models is that the probability of clicking on item $a$ at position $k$, $\bar{w}(a, k \mid x)$, depends only on the item and its position \cite{chuklin2015click}. Let
\begin{align*}
  \pi(a, k \mid x) = \sum_A \pi(A \mid x) \bOne\{a_k = a\}
\end{align*}
be the probability that item $a$ is displayed at position $k$ by policy $\pi$ in context $x$. Then a similar importance sampling trick to \eqref{eq:absCon} yields
\begin{align}
  \!\!\! V(h)
  & = \EE_{\bx} \left[\EE_{\bA \sim h(\cdot \mid \bx), \bw \sim D(\cdot \mid \bx)} [f(\bA, \bw)]\right]
  \label{eq:ip first} \\
  & = \EE_{\bx} \left[\EE_{\bA \sim h(\cdot \mid \bx)} \left[\sum_{k = 1}^K \bar{w}(\ba_k, k \mid \bx)\right]\right]
  \notag \\
  & = \EE_{\bx} \left[\sum_A h(A \mid \bx) \sum_{k = 1}^K \sum_{a \in E} \bar{w}(a, k \mid \bx) \bOne\{a_k = a\}\right]
  \notag \\
  & = \EE_{\bx} \left[\sum_{k = 1}^K \sum_{a \in E} \bar{w}(a, k \mid \bx)
  \sum_A h(A \mid \bx) \bOne\{a_k = a\}\right]
  \label{eq:ip independence} \\
  & = \EE_{\bx} \left[\sum_{k = 1}^K \sum_{a \in E} \bar{w}(a, k \mid \bx) \, h(a, k \mid \bx)\right]
  \label{eq:ip routine} \\
  & = \EE_{\bx} \left[\sum_{k = 1}^K \sum_{a \in E} \bar{w}(a, k \mid \bx) \, \pi(a, k \mid \bx)
  \frac{h(a, k \mid \bx)}{\pi(a, k \mid \bx)}\right]
  \label{eq:ip importance weight} \\
  & = \EE_{\bx} \left[\EE_{\bA \sim \pi(\cdot \mid \bx), \bw \sim D(\cdot \mid \bx)}
  \left[\sum_{k = 1}^K \bw(\ba_k, k) \frac{h(\ba_k, k \mid \bx)}{\pi(\ba_k, k \mid \bx)}\right]\right]\,,
  \label{eq:ip final}
\end{align}
where \eqref{eq:ip independence} is from our assumption that $\bar{w}(a, k \mid \bx)$ only depends on item $a$ and position $k$; \eqref{eq:ip importance weight} introduces the importance weight; and \eqref{eq:ip final} follows from identities \eqref{eq:ip first} to \eqref{eq:ip routine}, which are applied in the reverse order to $\pi$ instead of $h$.

Following the above derivation, we define the \emph{item-position (IP) estimator} as
\begin{align*}
  \vhat_\sip(h) =
  \frac{1}{|S|} \sum_{(x, A, w) \in S} \sum_{k = 1}^K w(a_k, k)
  \min \left\{\frac{h(a_k, k \mid x)}{\pihat(a_k, k \mid x)}, M\right\}
\end{align*}
for any estimated logging policy $\pihat$, logged dataset $S$, and clipping constant $M > 0$. Roughly speaking, the value of policy $h$ on logged dataset $S$ is the sum of clicks on logged item-position pairs scaled by their importance weights.

This estimator is expected to be more statistically efficient than $\vhat_\subl$ (\cref{sec:list estimator}) because it depends on fewer importance weights. In particular, the number of the weights in $\vhat_\sip$ is $O(K |E|)$ and the number of the weights in $\vhat_\subl$ is $O(|\Pi_K(E)|)$. The downside of the IP model, in fact of any model, is that it may not hold.

\subsection{Random Click Model (RCM)}
\label{sec:random estimator}

The random click model (RCM) is a variant of the IP model (\cref{sec:item-position estimator}) where the click probability is independent of both the item and its position, that is
\begin{align*}
  \bar{w}(a, k \mid x) = \bar{w}(a', k' \mid x)
\end{align*}
for any $a$, $a'$, $k$, $k'$, and $x$. This model is discussed in Section 3.1 of Chuklin \etal~\cite{chuklin2015click}. Under this model, $f(A, \bar{w}(\cdot \mid x))$ is independent of $A$ and therefore $V(h) = V(\pi)$ for any policy $h$. It follows that the value of $h$ can be estimated as
\begin{align*}
  \vhat_\text{Random}(h) =
  \frac{1}{|S|} \sum_{(x, A, w) \in S} \sum_{k = 1}^K w(a_k, k)\,,
\end{align*}
which is the average number of clicks collected by logging policy $\pi$. Although the above estimator is simplistic, it is hard to beat in practice when the responses of users do not change significantly with the policy. We return to this issue in \cref{sec:yandex bias}.

\subsection{Rank-Based Click Model (RCTR)}
\label{sec:position estimator}

The rank-based click model (RCTR) is also a variant of the IP model (\cref{sec:item-position estimator}) where the click probability is independent of the item, that is
\begin{align*}
  \bar{w}(a, k \mid x) = \bar{w}(a', k \mid x)
\end{align*}
for any $a$, $a'$, $k$, and $x$. This model is discussed in Section 3.2 of Chuklin \etal~\cite{chuklin2015click}. Under this model, the click probability can only depend on the position of the item. However, because all lists are displayed over the same $K$ positions, $f(A, \bar{w}(\cdot \mid x))$ is independent of $A$ and therefore $V(h) = V(\pi)$ for any policy $h$. It follows that the value of $h$ can be estimated as
\begin{align*}
  \vhat_\sr(h) =
  \frac{1}{|S|} \sum_{(x, A, w) \in S} \sum_{k = 1}^K w(a_k, k)\,,
\end{align*}
which is the average number of clicks collected by logging policy $\pi$. Note that this is the same estimator as $\vhat_\text{Random}$ in \cref{sec:random estimator}. Therefore, in the rest of the paper, we only refer to $\vhat_\sr$.

\subsection{Position-Based Click Model (PBM)}
\label{sec:pbm estimator}

The position-based click model (PBM) is a variant of the IP model (\cref{sec:item-position estimator}) where the probability of clicking on item $a$ at position $k$ factors as
\begin{align}
  \bar{w}(a, k \mid x) = \mu(a \mid x) \, p(k \mid x)\,,
  \label{eq:pbm model}
\end{align}
where $\mu(a \mid x)$ is the conditional probability of clicking on item $a$ in context $x$ given that its position is examined and $p(k \mid x)$ is the \emph{examination probability} of position $k$ in context $x$. This model was introduced as the \emph{examination hypothesis} in Craswell \etal~\cite{craswell2008} and is discussed in Section 3.3 of Chuklin \etal~\cite{chuklin2015click}. Joachims \etal~\cite{joachims2005accurately} showed that the examination probability of an item often depends heavily on its position. Note that the PBM is \emph{very different} from the rank-based click model in \cref{sec:position estimator}.

Suppose that the examination probabilities of all positions are known and let $p_x = (p(1 \mid x), \dots, p(K \mid x))$ be the vector of these probabilities. For any vectors $u$ and $v$, let $\langle u, v\rangle$ denote their dot product. Then the expected value of any policy $h$ in the PBM can be expressed as
\begin{align*}
  V(h)
  & = \EE_{\bx} \left[\sum_{a \in E} \mu(a \mid \bx) \sum_{k = 1}^K p(k \mid \bx) \, h(a, k \mid \bx)\right] \\
  & = \EE_{\bx} \left[\sum_{a \in E} \mu(a \mid \bx) \, \langle p_{\bx}, \pi(a, \cdot \mid \bx)\rangle
  \frac{\langle p_{\bx}, h(a, \cdot \mid \bx)\rangle}{\langle p_{\bx}, \pi(a, \cdot \mid \bx)\rangle}\right] \\
  & = \EE_{\bx} \left[\EE_{\bA \sim \pi(\cdot \mid x), \bw \sim D(\cdot \mid \bx)}
  \left[\sum_{k = 1}^K \bw(\ba_k, k)
  \frac{\langle p_{\bx}, h(\ba_k, \cdot \mid \bx)\rangle}{\langle p_{\bx}, \pi(\ba_k, \cdot \mid \bx)\rangle}\right]
  \right]\,,
\end{align*}
where the first equality is from identities \eqref{eq:ip first} to \eqref{eq:ip routine} in \cref{sec:item-position estimator} and our modeling assumption in \eqref{eq:pbm model}; the second equality introduces the importance weight; and the last equality is from identities \eqref{eq:ip first} to \eqref{eq:ip routine}, which are applied in the reverse order to $\pi$ instead of $h$.

Following the above derivation, we define the \emph{PBM estimator} as
\begin{align*}
  \vhat_\spbm(h) =
  \frac{1}{|S|} \sum_{(x, A, w) \in S} \sum_{k = 1}^K w(a_k, k)
  \min \left\{\frac{\langle p_x, h(a_k, \cdot \mid x)\rangle}
  {\langle p_x, \pihat(a_k, \cdot \mid x)\rangle}, M\right\}
\end{align*}
for any estimated logging policy $\pihat$, logged dataset $S$, and clipping constant $M > 0$.

This estimator is expected to be more statistically efficient than $\vhat_\sip$ (\cref{sec:item-position estimator}) because it depends on fewer importance weights. In particular, the number of the weights in $\vhat_\spbm$ is $O(|E|)$ and the number of the weights in $\vhat_\sip$ is $O(K |E|)$. The downside of the PBM is that it is more restrictive than the IP model.

\subsection{Document-Based Click Model (DCTR)}
\label{sec:item estimator}

The document-based click model (DCTR), which was introduced as the \emph{baseline hypothesis} in Craswell \etal~\cite{craswell2008}, assumes that the probability of clicking on an item depends only on its relevance, that is
\begin{align*}
  \bar{w}(a, k \mid x) = \bar{w}(a, k' \mid x)
\end{align*}
for any $a$, $k$, $k'$, and $x$. Note that this assumption can be viewed as a special case of \eqref{eq:pbm model}. In particular, it is equivalent to assuming that $p(k \mid x) = 1$ for any $k$ and $x$; or that $p_x = \bone_K$ for any $x$, where $\bone_K$ is a vector of all ones of length $K$. Therefore, the value of any policy $h$ in the DCTR can be estimated using $\vhat_\spbm$. In particular, it is
\begin{align*}
  \vhat_\si(h) =
  \frac{1}{|S|} \sum_{(x, A, w) \in S} \sum_{k = 1}^K w(a_k, k)
  \min \left\{\frac{\langle\bone_K, h(a_k, \cdot \mid x)\rangle}
  {\langle\bone_K, \pihat(a_k, \cdot \mid x)\rangle}, M\right\}
\end{align*}
for any estimated logging policy $\pihat$, logged dataset $S$, and clipping constant $M > 0$. In this work, we refer to this estimator as the \emph{item estimator}.

\subsection{Summary}
\label{sec:summary}

\begin{table}[t]
  \begin{tabular}{|l|l|} \hline
    Click model & Assumption \\ \hline
    RCM & $\bar{w}(a, k \mid x)$ is independent of both item $a$ and \\
    & position $k$ \\
    RCTR & $\bar{w}(a, k \mid x)$ only depends on position $k$ \\
    DCTR & $\bar{w}(a, k \mid x)$ only depends on item $a$ \\
    PBM & $\bar{w}(a, k \mid x) = \mu(a \mid x) \, p(k \mid x)$ \\ \hline
  \end{tabular}
  \vspace{0.1in}
  \caption{Dependence of click probabilities $\bar{w}(a, k \mid x)$ on item $a$ and its position $k$ in different click models.}
  \label{tab:click models}
\end{table}

We propose several offline estimators for the average number of clicks on lists of items generated by policy $h$. The main reason for studying multiple estimators is that the logged dataset $S$ in \eqref{eq:logged dataset} is finite. When $S$ is small, a simpler estimator may be more accurate because it depends on fewer importance weights, which can be estimated more accurately from less data. On the other hand, when $S$ is large, a more sophisticated estimator may be more accurate because it can capture all nuances of $S$. This is the so-called \emph{bias-variance tradeoff} and we demonstrate it empirically in \cref{sec:illustrative query}. The independence assumptions in our estimators are summarized in \cref{tab:click models}.

We refer to the \emph{item}, \emph{IP}, and \emph{PBM} estimators as being \emph{structured}, because their importance weights depend on individual items in the list. In comparison, the importance weights in the list estimator (\cref{sec:list estimator}) are at the level of the list. Our structured estimators are expected to use logged data more efficiently and we show this empirically in \cref{sec:experiments}. We analyze statistical properties of our estimators in the next section.

%% file: analysis.tex

\newcommand{\I}[1]{\mathds{1} \! \left\{#1\right\}}

\section{Analysis}
\label{sec:analysis}

In this section, we analyze our estimators from \cref{sec:estimators}. Our more general estimators in \cref{sec:weighted click estimators} can be analyzed similarly.

This section is organized as follows. In \cref{sec:more unbiased}, we show that our structured estimators are unbiased in a larger class of policies than the list estimator. In \cref{sec:lower bias}, we show that our structured estimators estimate the value of any policy with a lower bias than the list estimator. In \cref{sec:policy optimization}, we show that the best policy under our structured estimators has a higher value than that under the list estimator. All of our results are derived under the assumption that the corresponding click model holds.

\subsection{Unbiased in a Larger Class of Policies}
\label{sec:more unbiased}

All structured estimators in \cref{sec:estimators} are unbiased in a larger class of policies than the list estimator, under the assumptions that the logging policy is known, $\pihat = \pi$, and that the corresponding click model holds. We prove this for the IP estimator below.

\begin{proposition}
\label{prop:more unbiased} Fix any $M > 0$ and a class of policies $\cH$. Let
\begin{align*}
  \cH_\subl = \{h \in \cH: h(A \mid x) / \pi(A \mid x) \le M \text{ for all } A, x\}
\end{align*}
be the subset of policies where $\vhat_\subl$ is unbiased, the importance weights are not clipped for any $h \in \cH_\subl$. Let
\begin{align*}
  \cH_\sip = \{h \in \cH: h(a, k \mid x) / \pi(a, k \mid x) \le M \text{ for all } a, k, x\}
\end{align*}
be the subset of policies where $\vhat_\sip$ is unbiased, the importance weights are not clipped for any $h \in \cH_\sip$. Then in the IP model (\cref{sec:item-position estimator}), $\cH_\subl \subseteq \cH_\sip$ .
\end{proposition}
\begin{proof}
The proof follows from the observation that
\begin{align*}
  \frac{h(a, k \mid x)}{\pi(a, k \mid x)}
  & = \frac{\sum_{A:\, a_k = a} h(A \mid x)}{\sum_{A':\, a'_k = a} \pi(A' \mid x)} \\
  & = \sum_{A:\, a_k = a} \frac{h(A \mid x)}{\pi(A \mid x)} \frac{\pi(A \mid x)}{\sum_{A':\, a'_k = a} \pi(A' \mid x)} \\
  & \leq M
\end{align*}
holds for any $a$, $k$, and $x$.
\end{proof}

\noindent Similar claims can be derived analogously for both the item and PBM estimators. In particular, let $Y \in \{\si, \spbm\}$ and $\cH_Y \subseteq \cH$ be the subset of policies where $\vhat_Y$ is unbiased, where $\cH_Y$ is defined analogously to $\cH_\sip$. Then $\cH_\subl \subseteq \cH_\sip \subseteq \cH_Y$, under the assumption that the corresponding click model holds. We omit the proofs of these claims due to space constraints.

\subsection{Lower Bias in Estimating Policy Values}
\label{sec:lower bias}

All structured estimators in \cref{sec:estimators} estimate the value of any policy with a lower bias than the list estimator, under the assumptions that the logging policy is known, $\pihat = \pi$, and that the corresponding click model holds. We prove this for the IP estimator below.

\begin{proposition}
\label{prop:ip less biased} Fix any $M > 0$ and policy $h$. Then in the IP model (\cref{sec:item-position estimator}), the IP estimator $\vhat_\sip$ has a lower downside bias than the list estimator $\vhat_\subl$,
\begin{align*}
  \EE_{\bS} [\vhat_\subl(h)] \le
  \EE_{\bS} [\vhat_\sip(h)] \le
  V(h)\,.
\end{align*}
\end{proposition}
\begin{proof}
The second inequality follows from the observation that the clipping of importance weights leads to a downside bias. The first inequality is proved as follows. First, we note that
\begin{align*}
  & \EE_{\bS} [\vhat_\subl(h)] = \\
  & \quad \EE_{\bx} \left[ \sum_{a \in E} \sum_{k=1}^K \bar{w}(a, k \mid \bx)
  \sum_{A:\, a_k = a} \min\left\{h(A \mid \bx), M \pi(A \mid \bx)\right\} \right]\,, \\
  & \EE_{\bS} [\vhat_\sip(h)] = \\
  & \quad \EE_{\bx} \left[ \sum_{a \in E} \sum_{k=1}^K \bar{w}(a, k \mid \bx)
  \min\left\{h(a, k \mid \bx), M \pi(a, k \mid \bx)\right\} \right]\,.
\end{align*}
So, we can prove that $\EE_{\bS} [\vhat_\subl(h)] \le \EE_{\bS} [\vhat_\sip(h)]$ by showing that
\begin{align*}
  \sum_{A:\, a_k = a} \min\left\{h(A \mid x), M \pi(A \mid x)\right\} \le \min \left\{h(a, k \mid x), M \pi(a, k \mid x)\right\}
\end{align*}
holds for any $a$, $k$, and $x$. The above claim follows from
\begin{align*}
  h(a, k \mid x) =
  \hspace{-0.05in} \sum_{A:\, a_k = a} \hspace{-0.05in} h(A \mid x) \geq
  \hspace{-0.05in} \sum_{A:\, a_k = a} \hspace{-0.05in} \min\left\{h(A \mid x), M \pi(A \mid x)\right\}
\end{align*}
and
\begin{align*}
  M \pi(a, k \mid x) =
  \hspace{-0.05in} \sum_{A:\, a_k = a} \hspace{-0.05in} M \pi(A \mid x) \geq
  \hspace{-0.05in} \sum_{A:\, a_k = a} \hspace{-0.05in} \min\left\{h(A \mid x), M \pi(A \mid x)\right\}\,.
\end{align*}
This concludes our proof.
\end{proof}

\noindent Both the item estimator $\vhat_\si$ and PBM estimator $\vhat_\spbm$ are even less biased than the IP estimator $\vhat_\sip$.

\begin{proposition}
\label{prop:item less biased} Fix any $M > 0$ and policy $h$. Then in the DCTR (\cref{sec:item estimator}),
\begin{align*}
  \EE_{\bS} [\vhat_\subl(h)] \le
  \EE_{\bS} [\vhat_\sip(h)] \le
  \EE_{\bS} [\vhat_\si(h)] \le
  V(h)\,.
\end{align*}
\end{proposition}

\begin{proposition}
\label{prop:pbm less biased} Fix any $M > 0$ and policy $h$. Then in the PBM (\cref{sec:pbm estimator}),
\begin{align*}
  \EE_{\bS} [\vhat_\subl(h)] \le
  \EE_{\bS} [\vhat_\sip(h)] \le
  \EE_{\bS} [\vhat_\spbm(h)] \le
  V(h)\,.
\end{align*}
\end{proposition}

\noindent The above claims can be proved similarly to \cref{prop:ip less biased}. We omit their proofs due to space constraints.

\subsection{Policy Optimization}
\label{sec:policy optimization}

The estimators in \cref{sec:estimators} can be used to find better production policies. This section provides theoretical guarantees for finding such policies. We start with the list estimator in \cref{sec:list estimator}. Let
\begin{align*}
  \htilde_\subl = \argmax_{h \in \cH} \vhat_\subl(h)
\end{align*}
be the best policy according to the list estimator (\cref{sec:list estimator}). Then the value of $\htilde_\subl$, $V(\htilde_\subl)$, is bounded from below by the value of the optimal policy as follows.

\begin{theorem}
\label{thm:list} Let
\begin{align*}
  h_\subl^\ast =\argmax_{h \in \cH_\subl} V(h)
\end{align*}
be the best policy in the subset of policies $\cH_\subl$, which is defined in \cref{prop:more unbiased}. Then
\begin{align}
  & V(\htilde_\subl) \geq \label{eq:listopt} \\
  & V(h_\subl^\ast) - M \EE_{\bx} \left[F_\subl(\bx \mid \htilde_\subl) \right] -
  M \EE_{\bx} \left[ F_\subl(\bx \mid h_\subl^\ast) \right] - 2 K \sqrt{\frac{\ln(4/\delta)}{2|S|}} \nonumber
\end{align}
with probability of at least $1 - \delta$, where
\begin{align*}
  F_\subl(x \mid h) =
  \sum_A \I{\frac{h(A \mid x)}{\pi(A \mid x)} \le M} f(A, \bar{w}(\cdot \mid x)) \, \Delta(A \mid x)
\end{align*}
and $\Delta(A \mid x) = \abs{\pihat(A \mid x) - \pi(A \mid x)}$ is the error in our estimate of $\pi(A \mid x)$ in context $x$.
\end{theorem}

\noindent We prove our claim in \cref{sec:proofs}. Strehl \etal~\cite{strehl2010learning} proved a similar claim under the assumption that policies are deterministic given context. We generalize this result to stochastic policies.

Now we discuss the bound in \eqref{eq:listopt}. It contains three error terms, two expectations over $\bx$ and one $\sqrt{\log(1 / \delta)}$ term. The $\sqrt{\log(1 / \delta)}$ term is due to the randomness in generating the logged dataset. The two expectations are due to estimating the logging policy $\pi$ by $\pihat$. When the logging policy is known, $\pihat = \pi$, both terms vanish and our bound reduces to
\begin{align}
  V(\htilde_\subl) \ge V(h_\subl^\ast) - 2 K \sqrt{\frac{\ln(4/\delta)}{2|S|}}\,.
  \label{eq:list optimization lb}
\end{align}
The $\sqrt{\log(1 / \delta)}$ term vanishes as the size of the logged dataset $|S|$ increases.

The best policy under the list estimator, $\htilde_\subl$, is a solution to the following linear program (LP)
\begin{align*}
  \max_{h \in \cH, c} \quad & \sum_{(x, A, w) \in S} f(A, w) \, c(A \mid x) \\
  \text{s.t.} \quad & c(A \mid x) \, \pihat(A \mid x) \leq h(A \mid x)\,, && \forall A, x\,,\\
  & c(A \mid x) \leq M\,, && \forall A, x\,,\\
  & c(A \mid x) \geq 0\,, && \forall A, x\,,
\end{align*}
where $c$ is an auxiliary variable of the same dimension as policy $h$. The reason is that the maximization of $\displaystyle \min \left\{\frac{h(A \mid x)}{\pihat(A \mid x)}, M\right\}$ over $h(A \mid x)$ can be equivalently viewed as maximizing $c(A \mid x)$ subject to linear constraints $c(A \mid x) \, \pihat(A \mid x) \leq h(A \mid x)$ and $c(A \mid x) \leq M$. Note that although the number of lists $A$ is exponential in $K$, the above LP has $O(|S|)$ variables.

Similar guarantees can be obtained for all three structured estimators in \cref{sec:estimators}. Due to space constraints, we only analyze the value of the best IP policy (\cref{sec:item-position estimator}).

\begin{theorem}
\label{thm:ip} Let
\begin{align*}
  h_\sip^\ast = \argmax_{h \in \cH_\sip} V(h)\,, \quad \htilde_\sip & = \argmax_{h \in \cH} \vhat_\sip(h)\,,
\end{align*}
where $\cH_\sip$ is defined in \cref{prop:more unbiased}. Then in the IP model (\cref{sec:item-position estimator}),
\begin{align}
 & V(\htilde_\sip) \geq \label{eq:ipopt} \\
 & V(h_\sip^\ast) - M \EE_{\bx} \left[ F_\sip(\bx \mid \htilde_\sip) \right] -
 M \EE_{\bx} \left[ F_\sip(\bx \mid h_\sip^\ast) \right] - 2 K \sqrt{\frac{\ln(4/\delta)}{2|S|}} \nonumber
\end{align}
with probability of at least $1 - \delta$, where
\begin{align*}
  F_\sip(x \mid h) =
  \sum_{a \in E} \sum_{k=1}^K \I{\frac{h(a, k \mid x)}{\pi(a, k \mid x)} \le M} \bar{w}(a, k \mid x) \, \Delta(a, k \mid x)
\end{align*}
and $\Delta(a, k \mid x) = \abs{\pihat(a, k \mid x) - \pi(a, k \mid x)}$ is the error in our estimate of $\pi(a, k \mid x)$ in context $x$.
\end{theorem}

\noindent The theorem is proved in \cref{sec:proofs}. Similarly to the list estimator, the expectations over $\bx$ in \eqref{eq:ipopt} vanish when $\pihat = \pi$, and we get that
\begin{align}
  V(\htilde_\sip) \ge V(h_\sip^\ast) - 2 K \sqrt{\frac{\ln(4/\delta)}{2|S|}}\,.
  \label{eq:ip optimization lb}
\end{align}
By \cref{prop:more unbiased}, $\cH_\subl \subseteq \cH_\sip$. Therefore, the value of $h_\sip^\ast$ is at least as high as that of $h_\subl^\ast$, $V(h_\sip^\ast) \ge V(h_\subl^\ast)$. It follows from \eqref{eq:list optimization lb} and \eqref{eq:ip optimization lb} that the lower bound on the value of $\htilde_\sip$ is at least as high as that on $\htilde_\subl$.

The best policy under the IP estimator, $\htilde_\sip$, can be computed by solving a similar LP to that for $\htilde_\subl$. The number of variables in this LP is $O(|S|)$.

\subsection{Proofs for Section \ref{sec:policy optimization}}
\label{sec:proofs}

Before we prove \cref{thm:list}, we bound the expected value of the list estimator, $\vhat_\subl(h)$, for any policy $h$.

\begin{lemma}
\label{lem:list expect bound} Fix the estimated logging policy $\pihat$ and $M > 0$. Then
\begin{align*}
  \EE_{\bS}[\vhat_\subl(h)]
  & \le V(h) + M \EE_{\bx} \left[ F_\subl(\bx \mid h) \right]\,, \\
  \EE_{\bS}[\vhat_\subl(h)]
  & \ge \EE_{\bx} \left[ G_\subl(\bx \mid h) \right] - M \EE_{\bx} \left[ F_\subl(\bx \mid h) \right]\,,
\end{align*}
where $F_\subl(x \mid h)$ is defined in \cref{thm:list} and
\begin{align*}
  G_\subl(x \mid h) =
  \sum_A \I{\frac{h(A \mid x)}{\pi(A \mid x)} \le M} f(A, \bar{w}(\cdot \mid x)) \, h(A \mid x)\,.
\end{align*}
\end{lemma}
\begin{proof}
Note that
\begin{align*}
  & \EE_{\bS}[\vhat_\subl(h)] \\
  & \quad = \EE_{\bx} \left[\EE_{\bA \sim \pi(\cdot \mid \bx)} \left[f(\bA, \bar{w}(\cdot \mid \bx))
  \min \left\{ \frac{h(\bA \mid \bx)}{\hat{\pi}(\bA \mid \bx)}, M\right\}\right]\right]\,, \\
  & \quad = \EE_{\bx} \left[\sum_{A} \left[f(A, \bar{w}(\cdot \mid \bx))
  \min \left\{ \frac{h(A \mid \bx)}{\hat{\pi}(A \mid \bx)}, M\right\} \pi(A \mid \bx)\right]\right]\,.
\end{align*}
The main claims are obtained by bounding
\begin{align*}
  \min \left\{ \frac{h(A \mid \bx)}{\hat{\pi}(A \mid \bx)}, M\right\} \pi(A \mid \bx)
\end{align*}
from above and below in two cases, when $h(A \mid x) / \pi(A \mid x) \le M$ (\cref{lem:helper 1}) and when $h(A \mid x) / \pi(A \mid x) > M$ (\cref{lem:helper 2}).
\end{proof}

\begin{lemma}
\label{lem:helper 1} Let $h(A \mid x) / \pi(A \mid x) \le M$. Then
\begin{align*}
  \min \left\{ \frac{h(A \mid x)}{\pihat(A \mid x)}, M\right\} \pi(A \mid x)
  & \le h(A \mid x) + M \Delta(A \mid x)\,, \\
  \min \left\{ \frac{h(A \mid x)}{\pihat(A \mid x)}, M\right\} \pi(A \mid x)
  & \ge h(A\mid x) - M \Delta(A \mid x)\,.
\end{align*}
\end{lemma}

\begin{lemma}
\label{lem:helper 2} Let $h(A \mid x) / \pi(A \mid x) > M$. Then
\begin{align*}
  0 \le
  \min \left\{ \frac{h(A \mid x)}{\pihat(A \mid x)}, M\right\} \pi(A \mid x) \le
  h(A \mid x)\,.
\end{align*}
\end{lemma}

\noindent When the logging policy is known, $\pihat = \pi$, we have
\begin{align*}
  \EE_{\bx}\left[ G_\subl(\bx \mid h) \right] \le
  \EE_{\bS}[\vhat_\subl(h)] \le
  V(h)\,.
\end{align*}
In expectation, the list estimator underestimates $V(h)$. This is consistent with the intuition that clipping of the estimator leads to a downside bias. Also, when $\pihat = \pi$, $\EE_{\bS}[\vhat_\subl(h)] = V(h)$ for all $h \in \cH_\subl$, which means that the list estimator is unbiased for any policy $h$ that is not affected by the clipping.

Now we are ready to prove \cref{thm:list}.

\begin{proof}
From Hoeffding's inequality and the upper bound in \cref{lem:list expect bound},
\begin{align*}
  \vhat_\subl(\htilde_\subl) \le
  V(\htilde_\subl) + M \EE_{\bx} \left[ F_\subl(\bx \mid \htilde_\subl) \right] + K \sqrt{\frac{\ln(4/\delta)}{2|S|}}
\end{align*}
with probability of at least $1 - \delta / 2$. Similarly, from Hoeffding's inequality, the lower bound in \cref{lem:list expect bound}, and $\EE_{\bx} \left[ G_\subl(\bx \mid h_\subl^\ast) \right] = V(h_\subl^\ast)$ because $h_\subl^\ast \in \cH_\subl$,
\begin{align*}
  \vhat_\subl(h_\subl^\ast) \ge
  V(h_\subl^\ast) - M \EE_{\bx} \left[ F_\subl(\bx \mid h_\subl^\ast) \right] - K \sqrt{\frac{\ln(4/\delta)}{2|S|}}
\end{align*}
with probability of at least $1 - \delta / 2$. The final result follows from the observation that $\vhat_\subl(\htilde_\subl) \ge \vhat_\subl(h_\subl^\ast)$.
\end{proof}

\noindent Similarly to the above proof, the key step in the proof of \cref{thm:ip} are upper and lower bounds on the expected value of the IP estimator, which are presented below.

\begin{lemma}
\label{lem:ip expect bound} Fix the estimated logging policy $\pihat$ and $M > 0$. Then
\begin{align*}
  \EE_{\bS}[\vhat_\sip(h)]
  & \le V(h) + M \EE_{\bx} \left[ F_\sip(\bx \mid h) \right]\,, \\
  \EE_{\bS}[\vhat_\sip(h)]
  & \ge \EE_{\bx} \left[G_\sip(\bx \mid h)\right] - M \EE_{\bx}\left[F_\sip(\bx \mid h) \right]\,,
\end{align*}
where $F_\sip(x \mid h)$ is defined in \cref{thm:ip} and
\begin{align*}
  G_\sip(x \mid h) =
  \sum_{a\in E} \sum_{k=1}^K \I{\frac{h(a, k \mid x)}{\pi(a, k \mid x)} \le M} \bar{w}(a, k \mid x) \, h(a, k \mid x)\,.
\end{align*}
\end{lemma}
\begin{proof}
The proof follows the same line of reasoning as that in \cref{lem:list expect bound}, with the exception that we use inequalities
\begin{align*}
  \abs{ \min \left\{ \frac{ h(a, k \mid x) }{ \pihat(a, k \mid x)}, M\right\} \pi(a, k \mid x) - h(a, k \mid x)} \le
  M \Delta(a, k \mid x)
\end{align*}
when $h(a, k \mid x) / \pi(a, k \mid x) \le M$, and
\begin{align*}
  0 \le
  \min\left\{\frac{h(a, k \mid x)}{\pihat(a, k \mid x)}, M\right\} \pi(a, k \mid x) \le
  h(a, k \mid x)
\end{align*}
when $h(a, k \mid x) / \pi(a, k \mid x) > M$.
\end{proof}

\noindent Again, when the logging policy is known, $\pihat = \pi$, we have that
\begin{align*}
  \EE_{\bx}[G_\sip(\bx \mid h)] \le
  \EE_{\bS}[\vhat_\sip(h)] \le
  V(h)\,.
\end{align*}

%% file: generalization.tex

\section{Weighted Click Estimators}
\label{sec:weighted click estimators}

\begin{table}[t]
    \begin{tabular}{l} \hline
      \multicolumn{1}{c}{List estimator $\vhat_\subl(h)$} \\
      $\displaystyle \frac{1}{|S|} \sum_{(x, A, w) \in S} \sum_{k = 1}^K \theta_k w(a_k, k)
      \min \left\{\frac{h(A \mid x)}{\pihat(A \mid x)}, M\right\}$ \\ \hline
      \multicolumn{1}{c}{IP estimator $\vhat_\sip(h)$} \\
      $\displaystyle \frac{1}{|S|} \sum_{(x, A, w) \in S} \sum_{k = 1}^K \theta_k w(a_k, k)
      \min \left\{\frac{h(a_k, k \mid x)}{\pihat(a_k, k \mid x)}, M\right\}$ \\ \hline
      \multicolumn{1}{c}{RCTR estimator $\vhat_\sr(h)$} \\
      $\displaystyle \frac{1}{|S|} \sum_{(x, A, w) \in S} \sum_{k = 1}^K \theta_k w(a_k, k)$ \\ \hline
      \multicolumn{1}{c}{PBM estimator $\vhat_\spbm(h)$} \\
      $\displaystyle \frac{1}{|S|} \sum_{(x, A, w) \in S} \sum_{k = 1}^K \theta_k w(a_k, k)
      \min \left\{\frac{\langle\theta \circ p_x, h(a_k, \cdot \mid x)\rangle}
      {\langle\theta \circ p_x, \pihat(a_k, \cdot \mid x)\rangle}, M\right\}$ \\ \hline
      \multicolumn{1}{c}{Item estimator $\vhat_\si(h)$} \\
      $\displaystyle \frac{1}{|S|} \sum_{(x, A, w) \in S} \sum_{k = 1}^K \theta_k w(a_k, k)
      \min \left\{\frac{\langle\theta, h(a_k, \cdot \mid x)\rangle}
      {\langle\theta, \pihat(a_k, \cdot \mid x)\rangle}, M\right\}$ \\ \hline
  \end{tabular}
  \vspace{0.1in}
  \caption{Summary of our estimators. We denote by $u \circ v$ the entry-wise product of vectors $u$ and $v$.}
  \label{tab:estimators}
\end{table}

Suppose the reward of list $A$ is a weighted sum of clicks,
\begin{align*}
  f(A, w) = \sum_{k = 1}^K \theta_k w(a_k, k)
\end{align*}
for some fixed $\theta = (\theta_1, \ldots, \theta_K) \in \RR_{+}^K$. In \cref{sec:estimators}, we study the special case of $\theta = \bone_K$. Another important case is when $f(A, w)$ is the \emph{discounted cumulative gain (DCG)} of $A$, $\displaystyle \theta = \left(\frac{1}{\log_2(1 + k)}\right)_{k = 1}^K$. In this section, we generalize our estimators from \cref{sec:estimators} to any reward function of the above form.

Our generalized estimators are presented in \cref{tab:estimators}. These estimators are derived as follows. The IP and RCTR estimators are derived as in \cref{sec:item-position estimator,sec:position estimator}, respectively, with a minor difference that $\theta_k$ is carried with $\bar{w}(a, k \mid \bx)$ in all steps of the derivation. The PBM estimator is derived as in \cref{sec:pbm estimator}, with the difference that $\sum_{k = 1}^K \theta_k \, p(k \mid \bx) \, h(a, k \mid \bx)$ is rewritten as
\begin{align*}
  \langle\theta \circ p_{\bx}, \pi(a, \cdot \mid \bx)\rangle
  \frac{\langle\theta \circ p_{\bx}, h(a, \cdot \mid \bx)\rangle}
  {\langle\theta \circ p_{\bx}, \pi(a, \cdot \mid \bx)\rangle}\,,
\end{align*}
where $u \circ v$ is the entry-wise product of vectors $u$ and $v$. The item estimator is derived from the PBM estimator, as in \cref{sec:item estimator}.

%% file: experiments.tex

\section{Experiments}
\label{sec:experiments}

We experiment with the \emph{Yandex} dataset \cite{yandex}, which is a web search dataset with more than $167$ million web search queries. The dataset contains a training set, which is recorded over $27$ days, and a test set, which is recorded over $3$ days. Each \emph{record} in the Yandex dataset contains a query ID, the day when the query occurs, $10$ displayed items as a response to the query, and the corresponding click indicators of each displayed item. 

We observe in a majority of queries that the average number of clicks in the test set is significantly lower than in the training set, sometimes by an order of magnitude. This is due to the preprocessing of the Yandex dataset for the \emph{Personalized Web Search Challenge} \cite{yandex}. Due to this downside bias, all structured estimators in \cref{sec:estimators} perform extremely well at $M < 1$, when the training set is used as the logged dataset $S$ in \eqref{eq:logged dataset} and the test set is used to estimate $V(h)$. To avoid this systematic bias, which does not show the statistical efficiency of our estimators, we discard the test set and adopt a different evaluation methodology.

\subsection{Experimental Setup}
\label{sec:experimental setup}

We compare five estimators from \cref{sec:estimators}: list $\vhat_\subl$, RCTR $\vhat_\sr$, item $\vhat_\si$, IP $\vhat_\sip$, and PBM $\vhat_\spbm$. They are implemented as described in \cref{sec:estimators}. The examination probability of position $k$ in the PBM is set to $1 / k$. We leave its optimization for future work.

For each estimator, query $q$, and day $d \in [27]$, we put all records in day $d$ into the \emph{evaluation set} and all records in days $[27] \setminus \{d\}$ into the \emph{production set}. The production set is the logged dataset $S$ in \eqref{eq:logged dataset}. We estimate the \emph{production policy} by the frequencies of lists in the production set, and denote it by $\hat{\pi}_{q, d}$. We estimate the \emph{evaluated policy} by the frequencies of lists in the evaluation set, and denote it by $h_{q, d}$. Let $V_{q, d}$ be the value of $h_{q, d}$ on day $d$ in query $q$, which is estimated by its empirical average in the evaluation set; and $\hat{V}_{q, d}$ be its estimate from $\hat{\pi}_{q, d}$. We measure the error of the estimator in query $q$ by its average error over all evaluation sets, one for each day. In particular, we use the \emph{root-mean-square error (RMSE)} in
\begin{align*}
  \textstyle
  \sqrt{\frac{1}{27} \sum_{d = 1}^{27} (\hat{V}_{q, d} - V_{q, d})^2}\,.
\end{align*}
The error in multiple queries is defined as
\begin{align*}
  \textstyle
  \sqrt{\frac{1}{27 \abs{Q}} \sum_{q \in Q} \sum_{d = 1}^{27} (\hat{V}_{q, d} - V_{q, d})^2}\,,
\end{align*}
where $Q$ is the set of the evaluated queries.

All estimators are evaluated on three prediction problems: the expected number of clicks at the first $K = 2$ positions, where our dataset is restricted to those positions; the expected number of clicks at the first $K = 3$ positions, where our dataset is restricted to those positions; and the DCG, where the estimators are weighted as described in \cref{sec:weighted click estimators}. Our estimators yield only minor improvements in predicting the expected number of clicks at all positions. We discuss this issue in detail in \cref{sec:yandex bias}.

The queries in the Yandex dataset do not come with context. Therefore, we assume that the context is the same in all records.

\subsection{Illustrative Query}
\label{sec:illustrative query}

\begin{figure*}
  \centering
  \includegraphics[width = 0.32\textwidth]{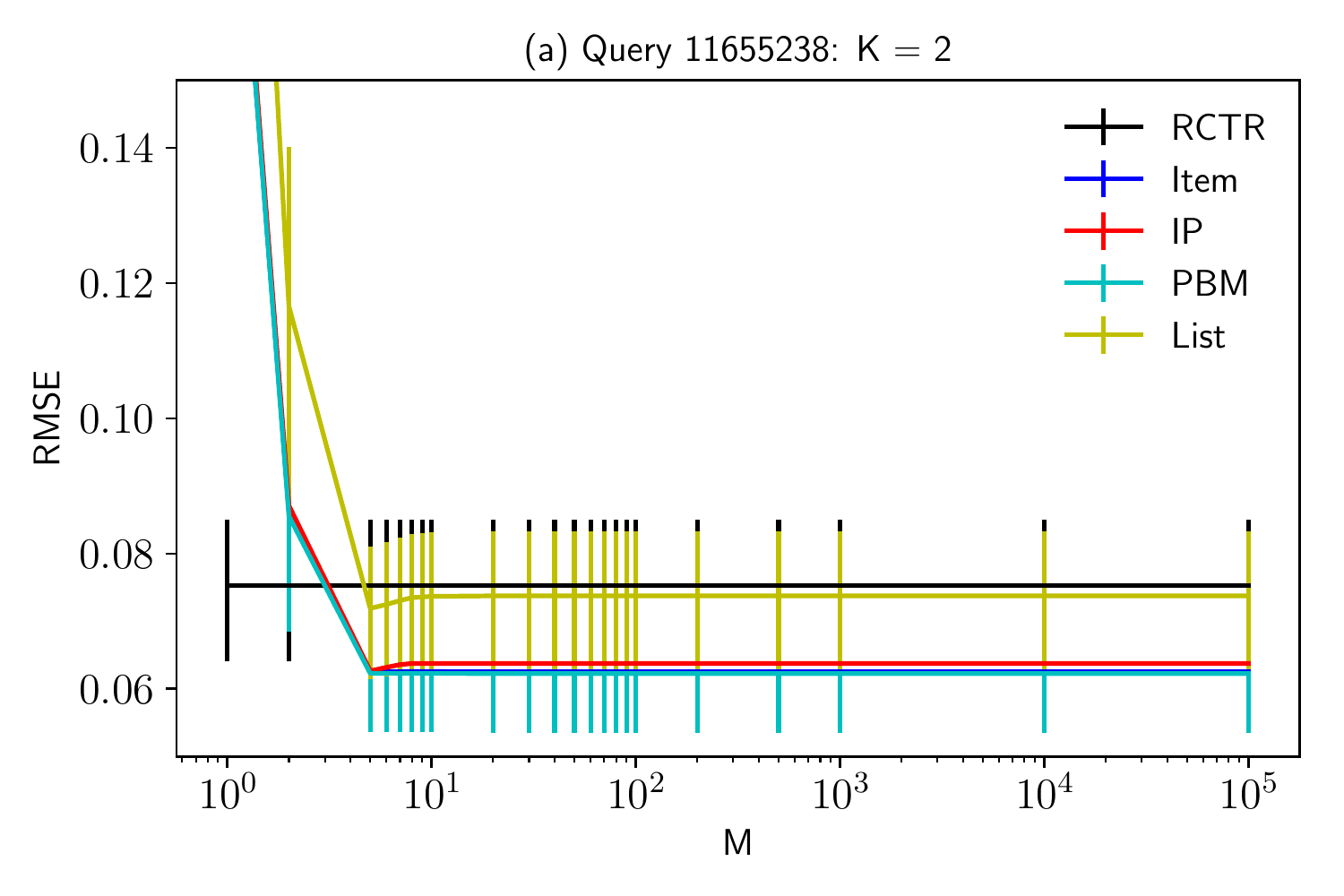}
  \includegraphics[width = 0.32\textwidth]{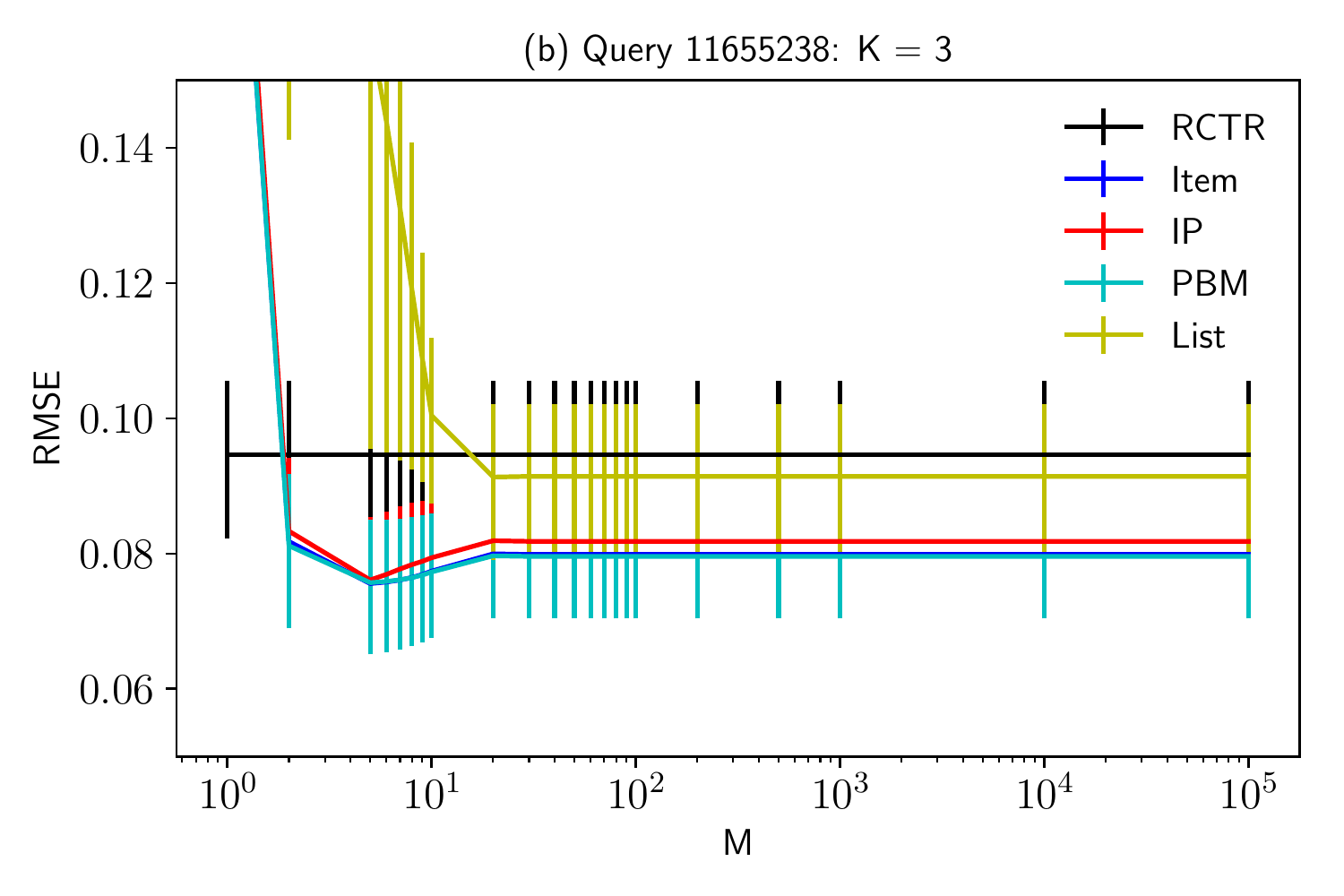}
  \includegraphics[width = 0.32\textwidth]{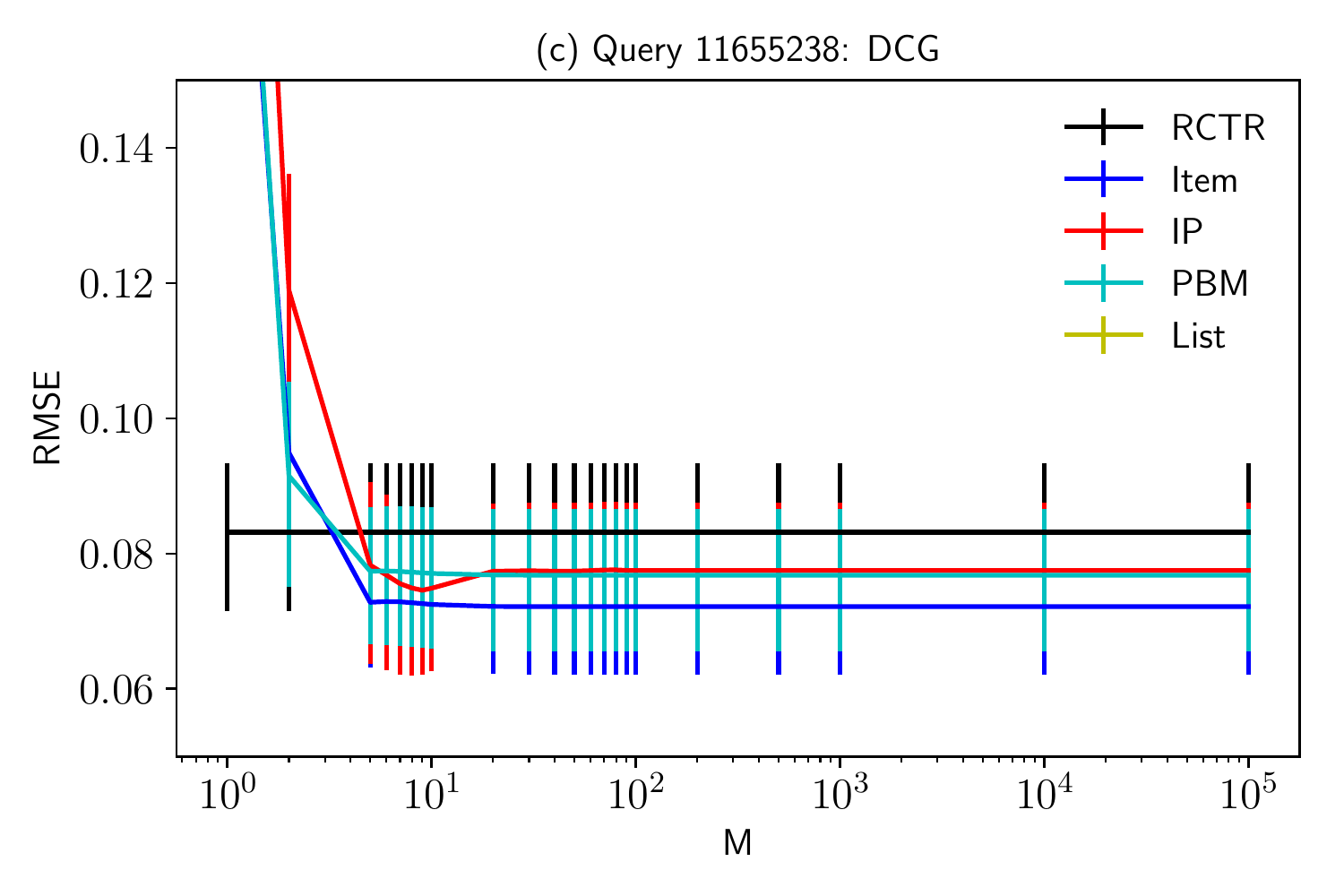}
  \caption{Prediction errors on query $11655238$ as a function of clipping parameter $M$.}
  \label{fig:query374}
\end{figure*}

This experiment illustrates our setup and its variations. It is conducted on query $11655238$, which appears in our dataset $3\,553$ times. The responses are $63$ distinct lists and $53$ distinct items.

The prediction errors at the first $K = 2$ positions are reported in \cref{fig:query374}a. The error of the RCTR estimator is $0.075$. For any $M \geq 40$, the errors of our three structured estimators are at least $16.07\%$ lower than that of the RCTR estimator and $13.55\%$ lower than that of the list estimator. The error of the list estimator at $M = 5$ is $4.50\%$ lower than that at $M = \infty$. The error of the IP estimator at $M = 5$ is $1.74\%$ lower than that at $M = \infty$. This shows the benefits of clipping in importance sampling estimators.

The prediction errors at the first $K = 3$ positions are reported in \cref{fig:query374}b. For any $M \geq 30$, the errors of our three structured estimators are at least $13.56\%$ lower than that of the RCTR estimator and $10.51\%$ lower than that of the list estimator. These gains further increase to $19.59\%$ and $51.73\%$ at $M = 5$.

The DCG prediction errors are reported in \cref{fig:query374}c. The errors of the list estimator never drop below $0.245$, and therefore are not visible in the figure. For any $M \geq 100$, the errors of our structured estimators are at least $6.82\%$ lower than that of the RCTR estimator and $68.34\%$ lower than that of the list estimator. These gains further increase to $10.32\%$ and $72.86\%$ at $M = 9$.

Finally, we observe in all figures that the item estimator consistently outperforms the IP estimator. We believe that this is because of the bias-variance tradeoff. In particular, the item estimator has a higher bias than the IP estimator because it is a special case of that estimator (\cref{sec:pbm estimator,sec:item estimator}). Because of that, it depends on fewer estimated importance weights, and can perform better in the regime of less training data, as in this query.

\subsection{Hundred Most Frequent Queries}
\label{sec:hundred most frequent queries}

\begin{figure*}
  \centering
  \includegraphics[width = 0.32\textwidth]{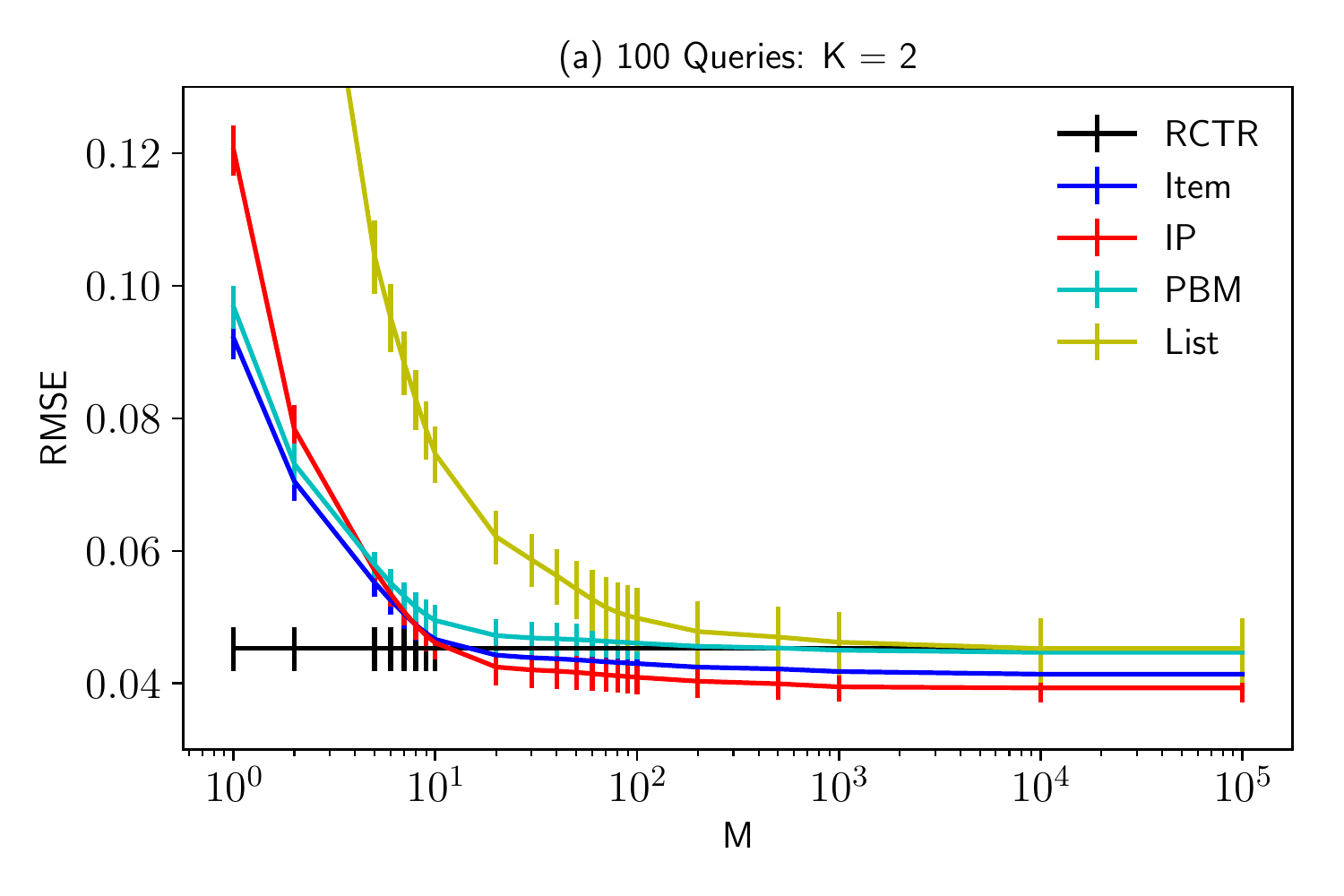}
  \includegraphics[width = 0.32\textwidth]{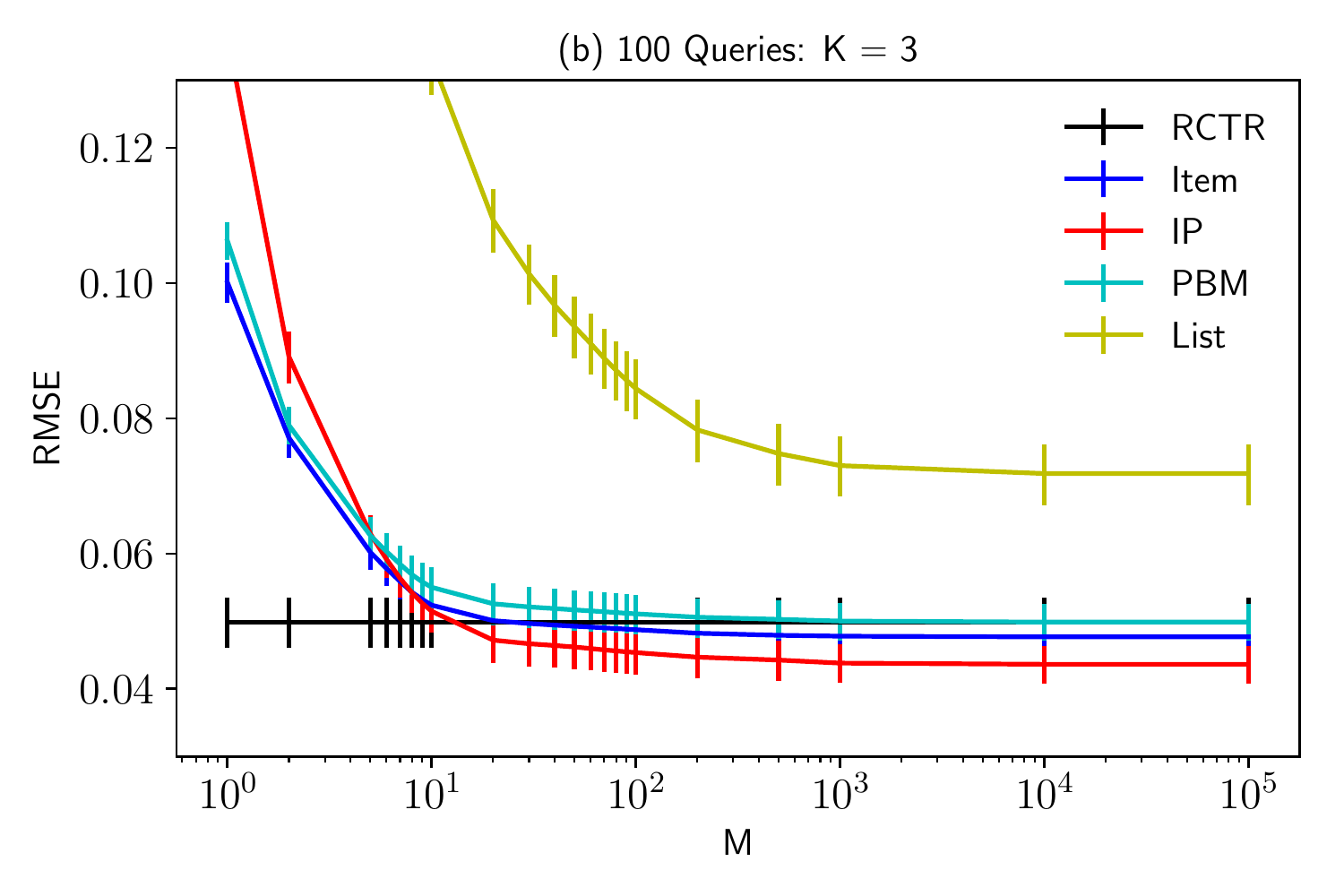}
  \includegraphics[width = 0.32\textwidth]{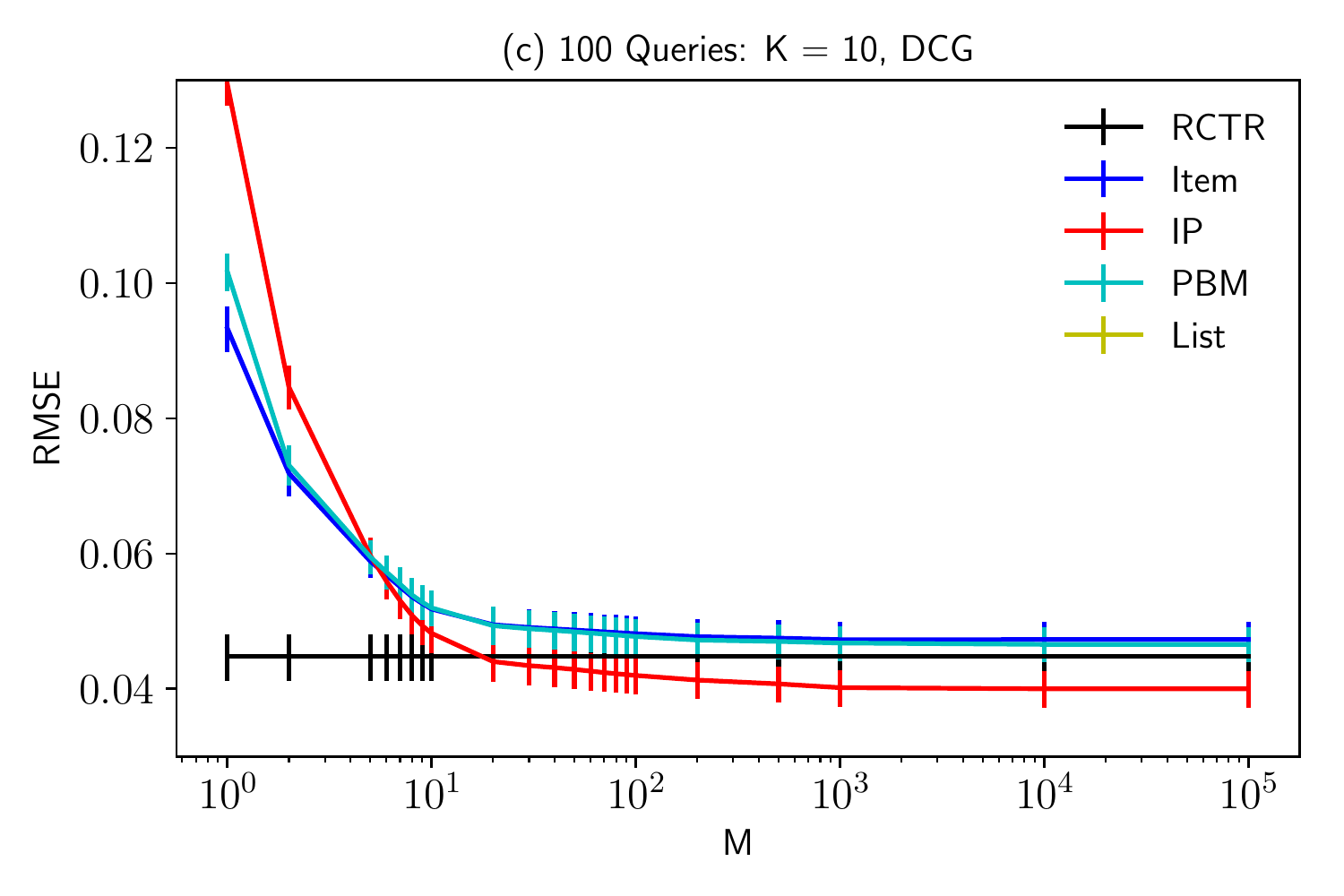}
  \caption{Prediction errors on $100$ most frequent queries as a function of clipping parameter $M$.}
  \label{fig:top100prob}
\end{figure*}

\begin{figure*}
  \centering
  \includegraphics[width = 0.32\textwidth]{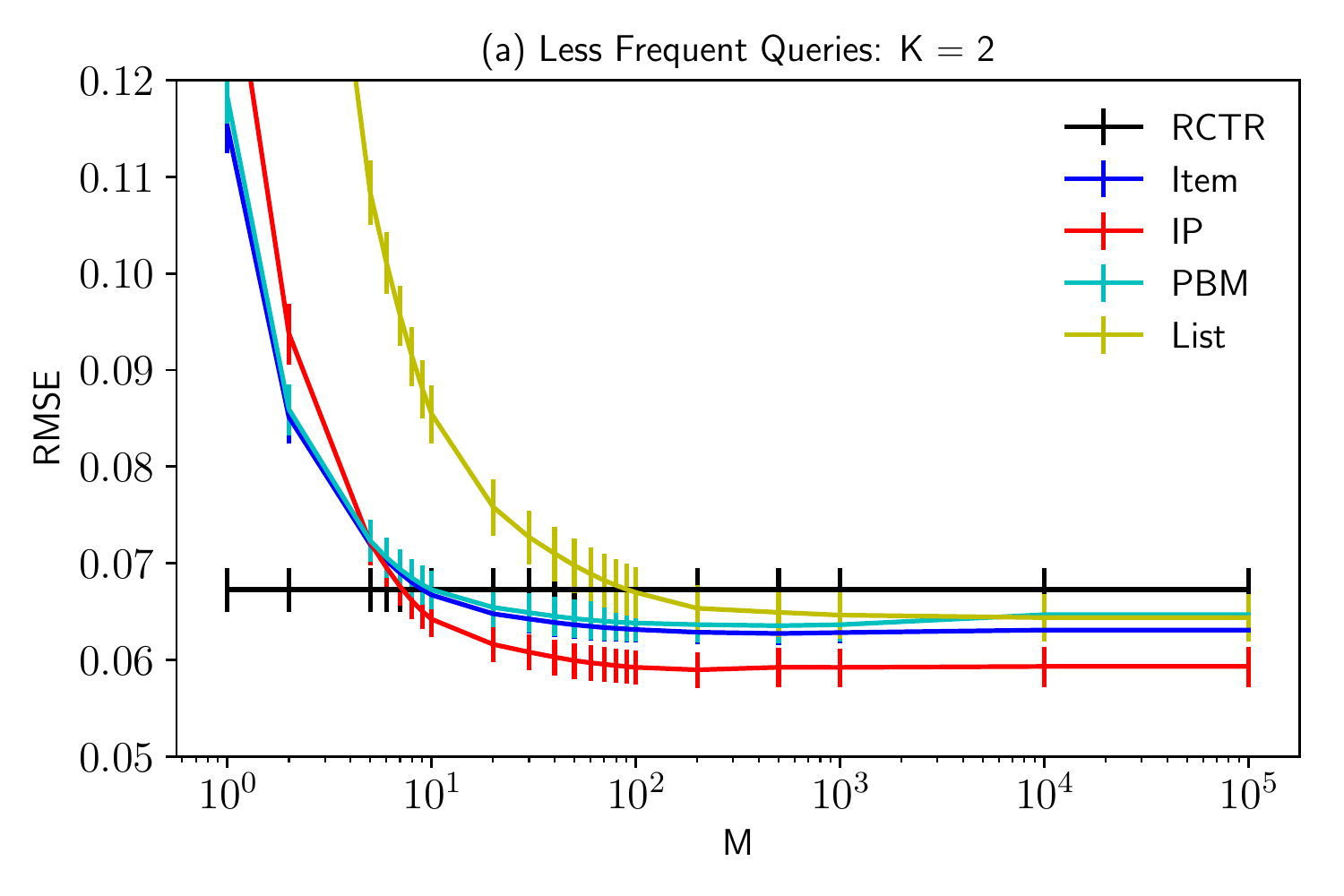}
  \includegraphics[width = 0.32\textwidth]{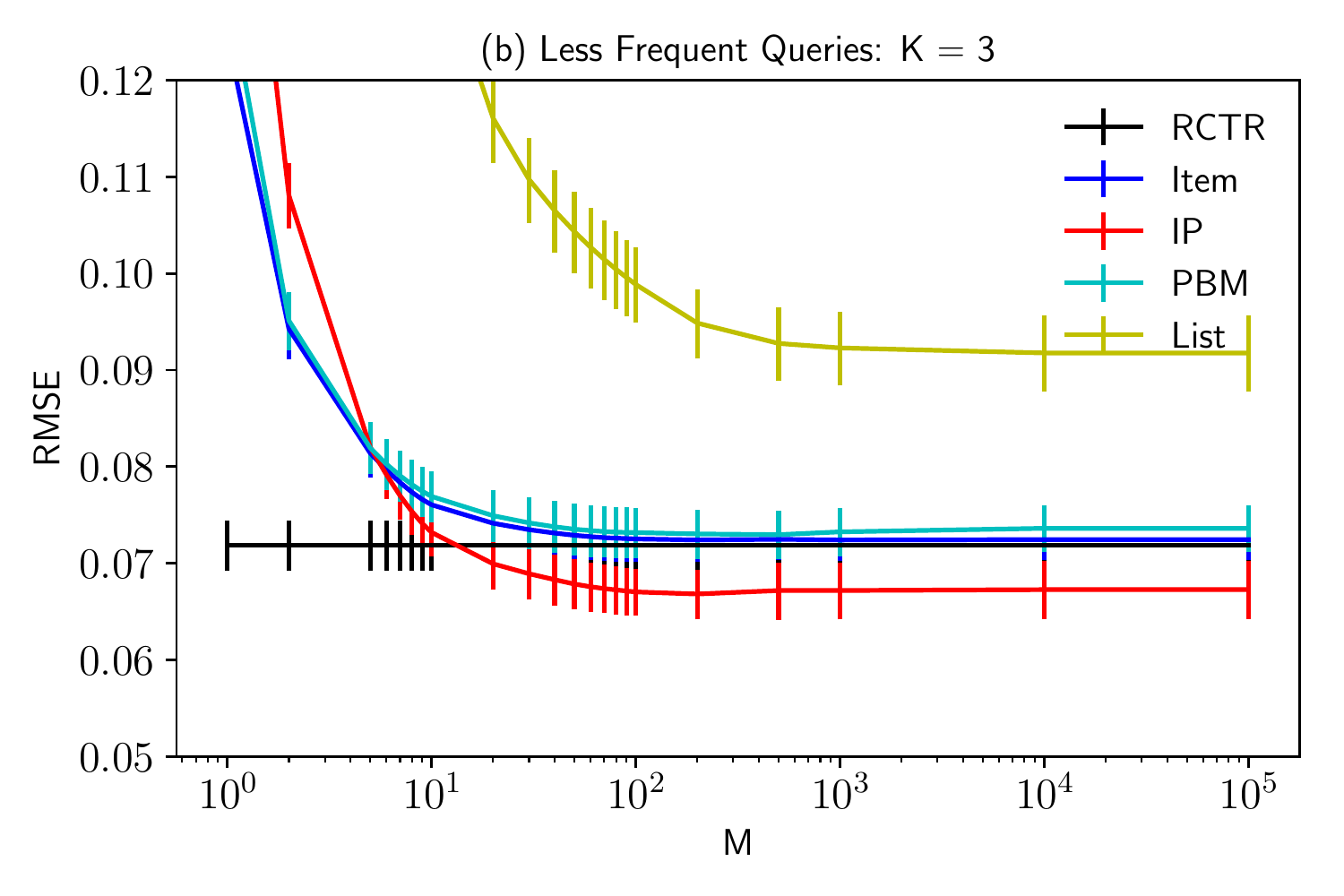}
  \includegraphics[width = 0.32\textwidth]{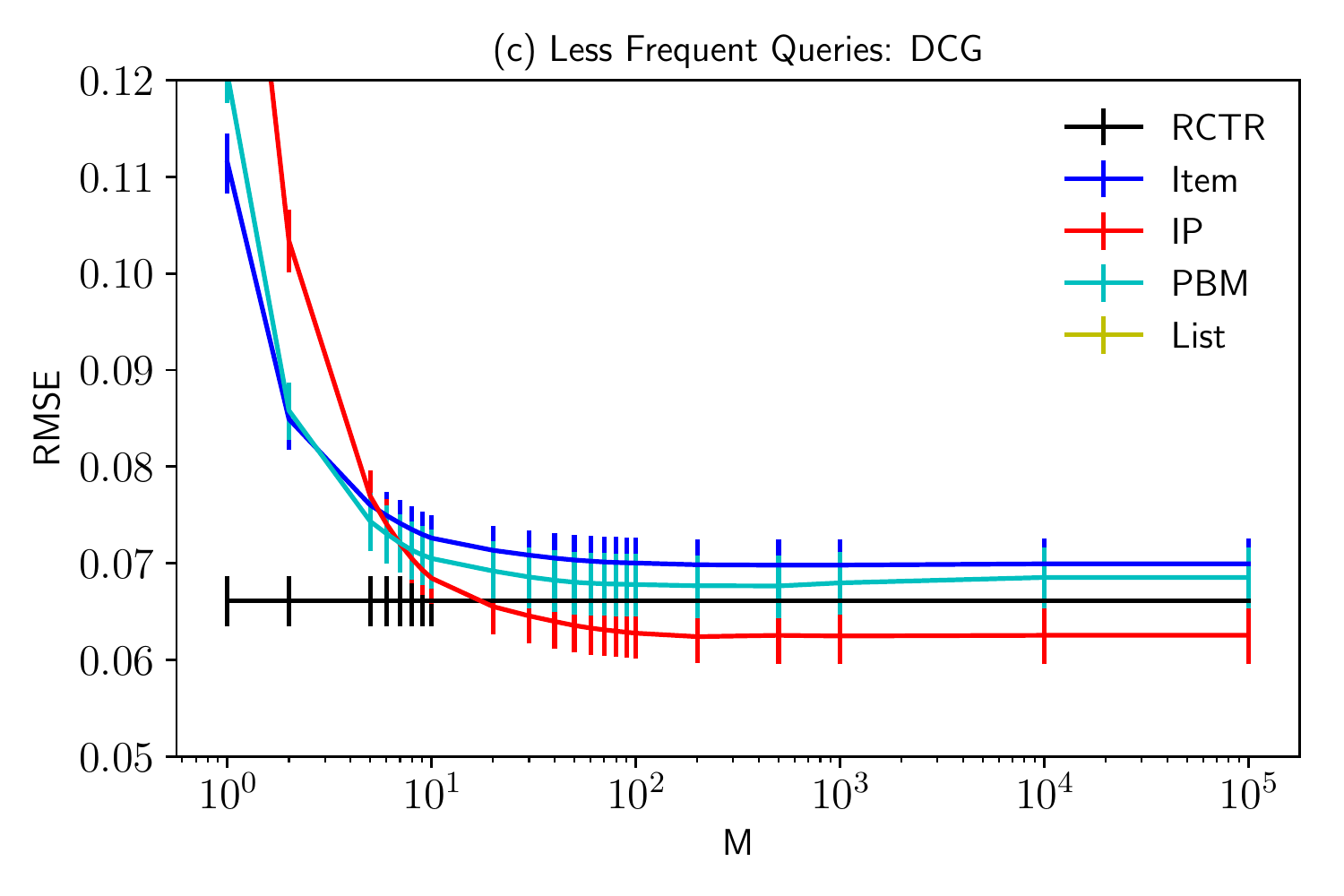}
  \caption{Prediction errors on less frequent queries as a function of clipping parameter $M$.}
  \label{fig:top10kprob}
\end{figure*}

Our second experiment is conducted on $100$ most frequent queries. The number of records in these queries ranges from $15\text{k}$ to $455\text{k}$, and the number of distinct lists ranges from $69$ to $10\text{k}$. This experiment validates our findings from \cref{sec:illustrative query} at a larger scale.

The errors of all estimators are reported in \cref{fig:top100prob}. The errors are averaged over all queries and days, as described in \cref{sec:experimental setup}. Similarly to \cref{sec:illustrative query}, we observe that our structured IP estimator outperforms both of our baselines. For any $M \geq 100$, the error of the IP estimator is at least $17.90\%$ ($K = 2$), $46.24\%$ ($K = 3$), $81.96\%$ (DCG) lower than that of the list estimator. The performance of the list estimator worsens dramatically from $K = 2$ to $K = 3$ because the number of distinct lists over three positions is typically much larger than over two. For any $M \geq 100$, the error of the IP estimator is at least $13.18\%$ ($K = 2$), $12.50\%$ ($K = 3$), $10.65\%$ (DCG) lower than that of the RCTR estimator.

\subsection{Less Frequent Queries}
\label{sec:less frequent queries}

Our last experiment is conducted on the tail $900$ queries from $1\text{k}$ most frequent queries. These queries are much less frequent than those in \cref{sec:hundred most frequent queries}, some with as few as $3\text{k}$ records.

The errors of all estimators are reported in \cref{fig:top10kprob}. We observe that the absolute errors of all estimators increase as we consider less frequent queries. This is expected since less frequent queries provide less training data. Nevertheless, our estimators still improve over baselines. In particular, the IP estimator improves consistently over both the RCTR and list estimators.

\subsection{Bias in the Yandex dataset}
\label{sec:yandex bias}

\begin{figure}
  \centering
  \includegraphics[width = 0.35\textwidth]{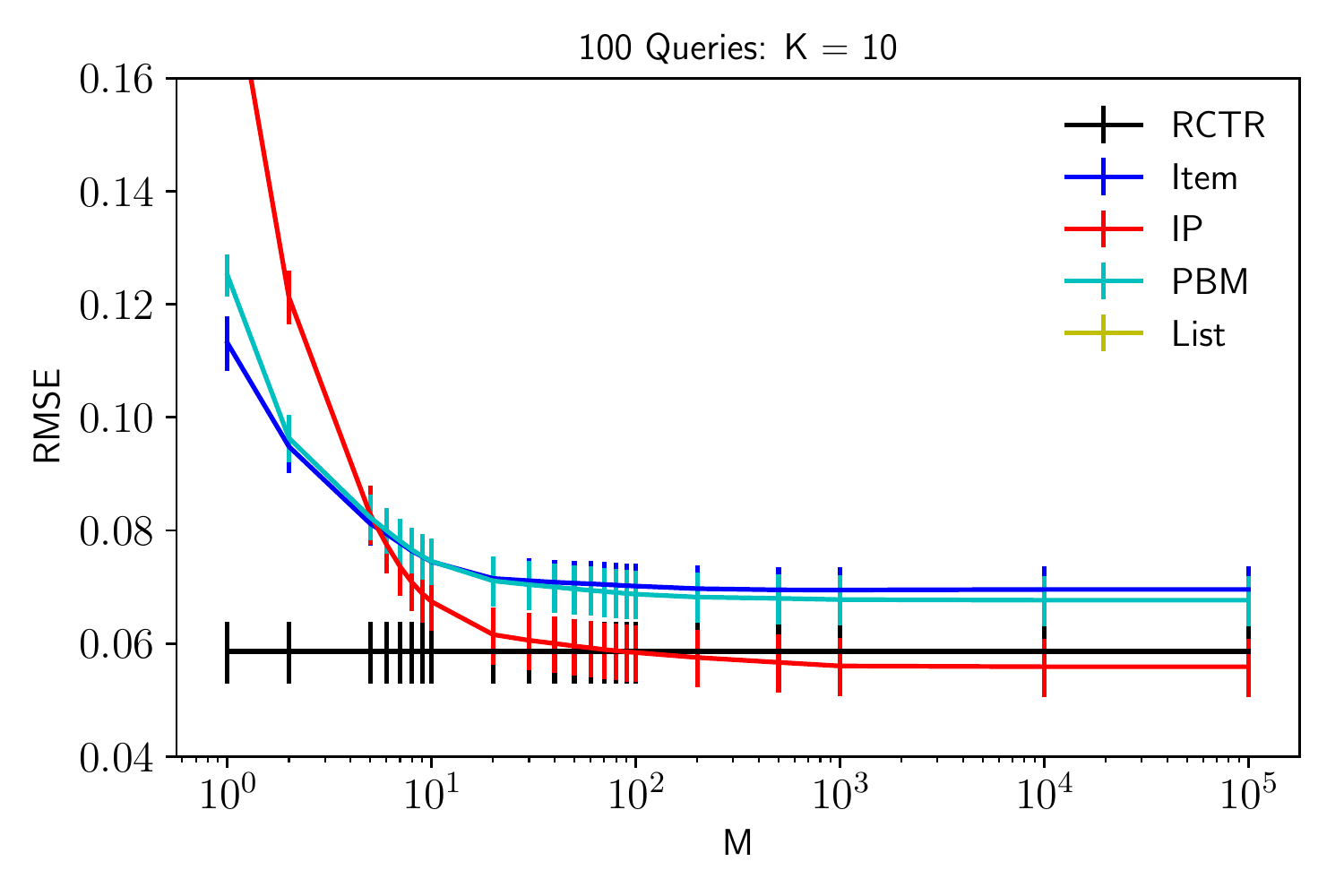}
  \caption{The errors in predicting the expected number of clicks at all $K = 10$ positions on $100$ most frequent queries, as a function of clipping parameter $M$.}
  \label{fig:100_k10}
\end{figure}

\begin{figure}
  \centering
  \includegraphics[width = 0.35\textwidth]{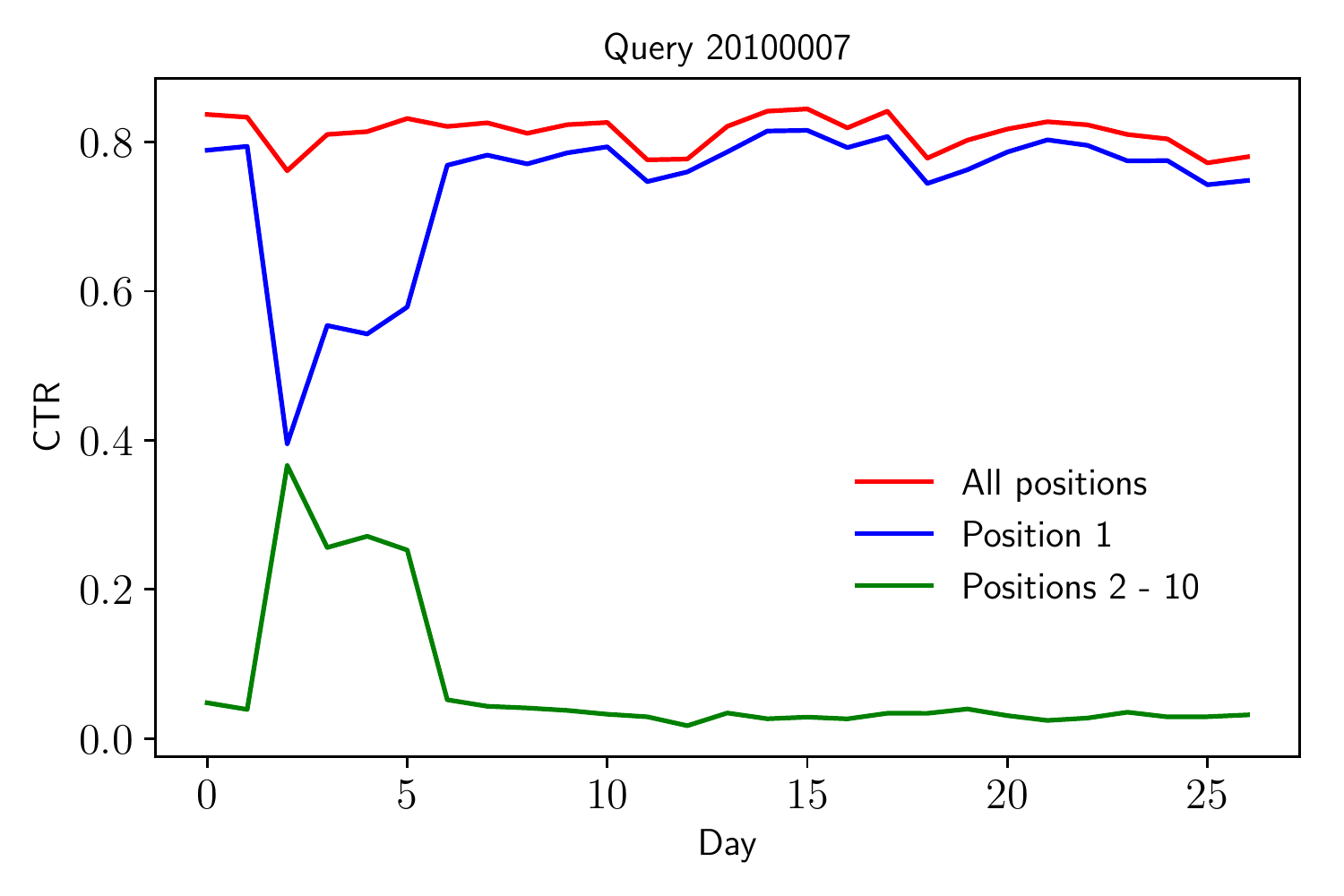}
  \caption{Expected number of clicks in query $20100007$ as a function of time, in days.}
  \label{fig:query7ctr}
\end{figure}

In our initial experiments, we estimated the expected number of clicks at all $K = 10$ positions. Our estimators performed poorly (\cref{fig:100_k10}). More specifically, only the IP estimator improved over the RCTR estimator, and that improvement was minimal.

We investigated this issue and found a very strong bias in the Yandex dataset, which we explain below. Most of our improvements over the RCTR baseline in the previous sections are due to queries whose responses change over time. One such query is shown in \cref{fig:query7ctr}, where the expected number of clicks at position $1$ drops between days $3$ and $5$. If the drop was due to an unattractive item that was placed at position $1$, the drop could be predicted if that item had been unattractive in the production set before. Also note that the expected number of clicks at all positions does not drop between days $3$ and $5$. Therefore, the RCTR estimator performs well at all positions. This changes when the positions are weighted, and we outperform it at $K = 2$, $K = 3$, and with the DCG.

The above study shows that interesting dynamics in data are necessary to outperform naive baselines. This is not surprising. If the expected number of clicks in the production and evaluation sets is similar, even simple estimators become strong baselines and we do not expect to outperform them.

Note that our estimators are data-driven, and require an overlap in the production and evaluated policies. Therefore, our approach is unsuitable for previously unseen queries. We also do not expect to perform well on infrequent queries.

%% file: related_work.tex

\section{Related Work}
\label{sec:related work}

The work described in the current paper is at the intersection of two areas. Reliable and efficient offline evaluation has been studied extensively in the bandit context, which we describe first. Then we discuss the prior work on click models, which are the starting point of our estimators, and have been shown to be representative of user behavior in various scenarios.

The problem of offline evaluation in the contextual bandit setting was first studied by Langford \etal~\cite{langford2008exploration} who provided an estimator for a stationary policy. This estimator used a variant of importance sampling and assumed that the policy does not depend on context. Many papers \cite{strehl2010learning,dudik2011doubly,li2011unbiased,nicol2014improving} followed by relaxing the assumptions of this work, improving robustness, and reducing the variance of offline evaluations. None but two considered the structure of lists.

Swaminathan \etal~\cite{swaminathan2016off} studied a similar click model to the IP model (\cref{sec:item-position estimator}) but with bandit feedback. In the context of lists, bandit feedback can be viewed as the total number of clicks on a list. Semi-bandit feedback, which we consider, are the indicators of clicks on each displayed item. The latter model of feedback is more informative, and Swaminathan \etal~\cite{swaminathan2016off} showed in their appendix that it can lead to better results, though they did not analyze their IP estimator in detail. We study the theoretical properties of several structured estimators and evaluated them at scale.

Early work in click-based evaluation in information retrieval \cite{joachims2005accurately,richardson2007} showed that higher ranks are more likely to be viewed and examined by users, and thus more likely to be clicked. Understanding and modeling of user behavior allows us to have evaluation methods that are more tolerant to the noise in behavioral data. Hofmann \etal~\cite{hofmann2016} conducted a comprehensive survey on efficient and reliable online evaluation of ranked lists. We focus on the offline evaluation aspect using click models.

Numerous click models have been proposed \cite{craswell2008,dupret2008,guo2009,chapelle2009}, including those that we introduce in \cref{sec:estimators}, and some models are comprehensive enough to explain finer details of user behavior. A generative model of clicks allows the evaluation of candidate ranking policies, and therefore reduce the dependence on expensive editorial judgments \cite{dupret2007}. Hofmann \etal~\cite{hofmann2012} used a similar importance sampling driven method to leverage historical comparisons of ranking policies to predict the outcomes of future comparisons. Click models usually have latent variables. Therefore, their parameters are estimated with an iterative EM-like procedure that lacks theoretical guarantees. In this work, we do not explicitly fit a click model. We only use the structure of the click model to represent the same independence assumptions as those in that model.

The most relevant related area to this paper is unbiased learning-to-rank and evaluation from logged data. Joachims \etal~\cite{joachims2017unbiased} use the sum of ranks of relevant items as a metric and Wang \etal~\cite{wang2018position} use the precision of clicked items, while we consider more general reward functions. They both focus on the PBM, which is only one instance of a broad class of click models. Our experimental results show that the PBM is not necessarily the best model for offline evaluation. Though we focus on evaluation \cite{gilotte2018offline}, we show in \cref{sec:policy optimization} that our estimators can be used for policy optimization. Combining these estimators with models trained by batch offline processes for the many learning-to-rank objectives \cite{radlinski2008does} is an interesting future direction.

Some works in information retrieval also consider user models to design metrics for ranked lists \cite{chuklin2013click,moffat2008rank,chapelle2009expected,carterette2011system,wang2009pskip,yilmaz2010expected}. Those papers do not consider the counterfactual imbalance between logged data and a new production policy, and have no theoretical guarantees in our setting.

%% file: conclusions.tex

\section{Conclusions}
\label{sec:conclusions}

We propose various estimators for the expected number of clicks on lists generated by ranking policies that leverage the structure of click models. We prove that our estimators are better than the unstructured list estimator, in the sense that they are less biased and have better guarantees for policy optimization. They also consistently outperform the list estimator in our experiments.

Our work can be extended in multiple directions. For instance, our key assumption is that the reward function $f(A, w)$ is linear in the contributions of individual items in $A$. Such functions cannot model many interesting non-linear metrics, such as the indicator of at least one click. Another potential direction for extending our work is to generalize it to click models with partial observations, such as the cascade model \cite{craswell2008}. The main challenge in the cascade model is that the item may not be clicked due to more attractive higher-ranked items, not because it is unattractive. This phenomenon is not captured by any of our estimators.

Our estimators need to be evaluated better empirically. To the best of our knowledge, the Yandex dataset is the only public click dataset that is both large-scale and comprises clicks on individual items in recommended lists. Therefore, a better evaluation could not be done in this paper.

We also want to comment on the generality of our result. Since the reward function $f(A, w)$ is linear in the contributions of individual items in $A$ and we do not make independence assumptions on the entries of $\bw \sim D(\cdot \mid x)$, our work solves the problem of offline evaluation in stochastic combinatorial semi-bandits \cite{gai12combinatorial,chen13combinatorial,kveton15tight,wen15efficient}. Therefore, our methods can be used to estimate the values of paths in graphs from semi-bandit feedback, for instance.